\documentclass[letterpaper,10pt,conference]{ieeeconf} 

\IEEEoverridecommandlockouts

\usepackage{amsmath}
\usepackage{mathtools}
\usepackage{setspace}
\usepackage{lpform}
\usepackage{textcomp}
\usepackage{algorithmic}
\usepackage[english]{babel}
\usepackage{amssymb}
\usepackage{graphicx}
\usepackage{xcolor}
\usepackage{enumerate}
\usepackage{framed}
\usepackage{cite}
\usepackage{tabularx}
\usepackage[style=2,innerleftmargin=5pt,innerrightmargin=5pt,skipabove=5pt]{mdframed}

\usepackage{caption}
\usepackage{multirow}
\usepackage{pdflscape}
\usepackage{booktabs}
\usepackage{siunitx}

\newtheorem{definitionenv}{\bf Definition}
\newtheorem{lemmaenv}{Lemma}
\newtheorem{theoremenv}{Theorem}
\newtheorem{corollaryenv}{Corollary}

\newcommand{\gobble}[1]{}

\newenvironment{definition}[1][\unskip]{\vspace{0.05in}\begin{definitionenv}\em}{\end{definitionenv}\vspace{0.05in}}
\newenvironment{lemma}{\vspace{0.05in}\begin{lemmaenv}\em}{\end{lemmaenv}\vspace{0.05in}}

\newcommand*{\probleminternal}[4]{
	\begin{mdframed}\parbox{0.98\columnwidth}{
			\textbf{#4: #1} \\[0.05in]
			\renewcommand{\tabcolsep}{2pt}
			\begin{tabularx}{\linewidth}{rX}
				\emph{Input:} & #2 \\
				\emph{Output:} & #3
			\end{tabularx}
		}\end{mdframed}
		\par
	}

	\newcommand*{\casestudyinternal}[5]{
	\begin{mdframed}\parbox{0.98\columnwidth}{
			\renewcommand{\tabcolsep}{2pt}
			\begin{tabularx}{\linewidth}{lX}
				\textbf{Goal:} #1 \\
				\textbf{Soft constraints:} \\ 
				#2 \\
				\textbf{Solution:} \\
				#3 \\
				\textbf{Satisfied constraints:}  #4 \\
				\textbf{Computation time:}  #5 \\
			\end{tabularx}
		}\end{mdframed}
		\par
	}

\newcommand*{\problem}[3]{\probleminternal{#1}{#2}{#3}{Problem}}

\newcommand{\trace}{\operatorname{trace}}

\newcommand{\Words}{\operatorname{Words}}
\newcommand{\T}{\operatorname{\mathrm{T}}}
\newcommand{\F}{\operatorname{\mathrm{F}}}

\makeatletter
\newcommand*{\Relbarfill@}{\arrowfill@\Relbar\Relbar\Relbar}
\newcommand*{\xeq}[2][]{\ext@arrow 0055\Relbarfill@{#1}{#2}}
\makeatother

 \providecolor{added}{rgb}{0,0,0}
 \providecolor{changed}{rgb}{0,0,0}
 \providecolor{deleted}{rgb}{1,0,0}
 \newcommand{\added}[1]{{\color{added}{}#1}}
 \newcommand{\changed}[1]{{\color{changed}{}#1}}

\usepackage[norelsize, linesnumbered, ruled, lined, boxed, commentsnumbered]{algorithm2e}

\SetKwRepeat{Do}{do}{while}%

\begin{document}
\title{What to Do When You Can't Do It All:\\{}Temporal Logic Planning with Soft Temporal Logic Constraints}

\author{Hazhar Rahmani \qquad Jason M. O'Kane\thanks{%
    The authors are with the Department of Computer Science and Engineering at the University of South Carolina.
    This material is based upon work supported by the NSF under Grant Nos.~1526862 and 1849291.}}

\maketitle
\begin{abstract}
  In this paper, we consider a temporal logic planning problem in which the objective is to find 
  an infinite trajectory that satisfies an optimal selection from a set of soft specifications expressed in linear temporal logic (LTL) while nevertheless satisfying a hard specification expressed in LTL.
  Our previous work considered a similar problem in which linear dynamic logic for finite traces (LDL$_f$), rather than LTL, was used to express the soft constraints.
  In that work, LDL$_f$ was used to impose constraints on finite prefixes of the infinite trajectory.
  By using LTL, one is able not only to impose constraints on the finite prefixes of the trajectory, but also to set `soft' goals across the entirety of the infinite trajectory.
  Our algorithm first constructs a product automaton, on which the planning problem is reduced to computing a lasso with minimum cost.
  Among all such lassos, it is desirable to compute a shortest one.
  Though we prove that computing such a shortest lasso is computationally hard, we also introduce an efficient greedy approach to synthesize short lassos nonetheless.
  We present two case studies describing an implementation of this approach, and report results of our experiment comparing our greedy algorithm with an optimal baseline.
\end{abstract}
  
\section{Introduction}
\label{sec:intro}
%
%Temporal logics have become one of the most powerful and expressive tools for planning in robotics~\cite{fainekos2009temporal, smith2011optimal, schillinger2018simultaneous, he2015towards, ulusoy2013optimality, lahijanian2011temporal}. 
%
%Temporal logics have become one of the most powerful and expressive tools for planning in robotics~\cite{fainekos2009temporal, smith2011optimal, schillinger2018simultaneous, he2015towards, ulusoy2013optimality}. 
%
Temporal logics have become one of the most powerful and expressive tools for planning in robotics~\cite{fainekos2009temporal, schillinger2018simultaneous, he2015towards, ulusoy2013optimality}. 
%
%Temporal logics have become one of the most powerful and expressive tools for planning in robotics(~\cite{fainekos2009temporal, schillinger2018simultaneous, he2015towards, ulusoy2013optimality}, to name but a few). 
%
Such logics, including linear temporal logic (LTL) specifically, offer high-level, user-friendly languages for specifying complex missions and tasks.
%, and are coupled with automata- and graph-theoretic approaches that automate generating plans to achieve those missions and tasks.
%
In fact, as simple and intuitive as temporal logic is for humans to understand, it is also precise and rigorous for robot algorithms to manipulate.
%
%By using temporal logics as a specification language, the user needs not to be concerned about the details of how a plan is c
In particular, temporal logic has disrupted the classical conception of motion and path planning ---which deals with making a finite, point to point trajectory that avoids obstacles--- by allowing the imposition of other kinds of temporal or spatial constraints and by allowing the robot to make infinite, rather than only finite, trajectories.
%

%This paper is in the same thread of adding more extensions to temporal logic planning. 
%
In this paper, we consider a temporal logic planning problem in which a robot is tasked to accomplish a mission specified by an LTL formula while optimally satisfying a set of additional, possibly conflicting, LTL formulas.  These extra constraints could be user preferences, safety rules, soft goals, or other constraints.  

To illustrate the setting, see Figure~\ref{fig:intro}, in which a social enrichment robot, capable of making animal balloons and juggling, visits the residents of a retirement home.
The robot's basic mission is to visit the two common rooms, each infinitely often.  In addition to this basic mission, however, the robot is also charged with satisfying a collection of soft constraints, given in order of their relative importance. For example, we might prefer to maintain fairness by ensuring, if possible, that after making animal balloons in room 1, it should also do the same act in room 2.
Or perhaps the manager wants the robot to eventually perform juggling in room 2, if its current act in that room is making animal balloons.
The residents of room 2 might even prefer not to see the balloon animal act at all.  The essence of our problem is to determine how the robot can act, to satisfy its primary mission, along with some optimal subset of these kinds of soft constraints.

%These constraints could be prioritized either by the manager or by the robot itself based on some criteria. 
%
%In this example, all the constraints cannot be satisfied together, and thus, the robot chooses to violate constraint (3) rather than one of the constraint (1) or (2), which, if are applied together, are not consistent with constraint (3).
%

\begin{figure}[t]
  \centering
  \includegraphics[width=\linewidth]{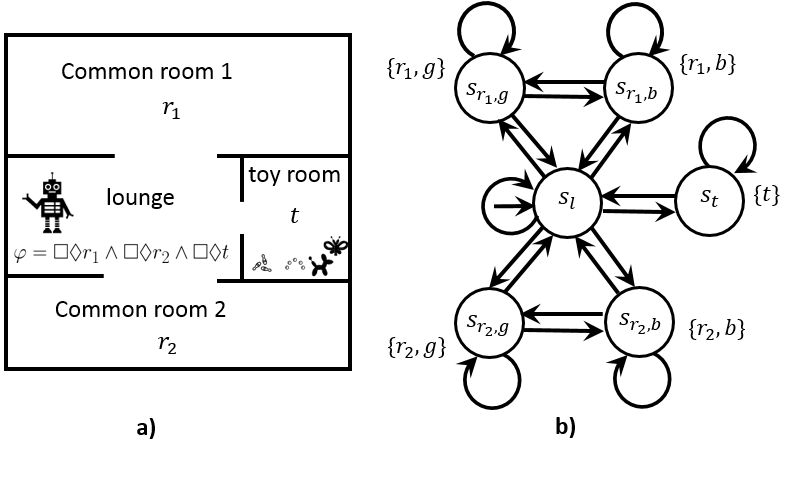}
  
  \caption{
    \textbf{a)} A retirement home in which a social enrichment robot visits each of the common rooms to perform juggling and to make animal balloons.
    Its primary mission is expressed by the LTL formula $\varphi = \square \lozenge r_1 \wedge \square \lozenge r_2 \wedge \square \lozenge t$.
    \textbf{b)} A transition system that models the robot's state within this environment.
  }
   \label{fig:intro} 
\end{figure}

Our prior work~\cite{rahmani2019optimal} considered a related problem in which the soft constraints were expressed in linear dynamic logic for finite traces (LDL$_f$)~\cite{de2013linear}.  Such formulas can express constraints only on finite prefixes of the trajectory, rather than on the entire trajectory as a whole.
%
%A limitation of that work is that LDL$_f$ soft constraints cannot express certain kinds of soft goals---e.g., \changed{\emph{liveness properties}, which can be violated only in an infinite (rather than a finite) execution, asserting that ``something good will eventually happen''.}
\changed{A limitation of that work is that LDL$_f$ soft constraints cannot express soft goals that are satisfied only by infinite (rather than finite) trajectories. As an example, a task that requires the social enrichment robot to infinitely often perform the act of juggling is a simple soft goal that cannot be expressed by LDL$_f$.}
%A limitation of that work is that LDL$_f$ soft constraints cannot express certain kinds of soft goals---e.g., \emph{liveness properties}~\cite{alpern1987recognizing}, asserting that something good should happen.
%soft specifications (2) and (4) cannot be expressed using LDL$_f$.
%Informally, a property is liveness if it is violated in infinite (rather than finite) time.
%
The difference in the language used to express the soft constraints not only improves the expressivity of the approach, but it leads to significant (and new, compared to the LDL$_f$ case) algorithmic challenges.

% In the planning problem in which we are interested, the robot must make a trajectory that satisfies the hard specification and as many as possible of the soft constraints with the consideration that less prior constraints should not be satisfied if they violate higher prior constraints.

%
%
% 
%Although by allowing the soft constraints to be expressed in LTL rather than LDL$_f$ one will be able to impose constraints not only on finite prefixes of the trajectory but also on infinity time instance of the trajectory, finding a solution when soft constraints are in LTL is more difficult.

We contribute in this paper, a general formulation of the kind of problem in Figure~\ref{fig:intro}. 
To do so, we first review related work in Section~\ref{sec:rel}, and then present our problem statement in Section~\ref{sec:def}.
In Section~\ref{sec:alg}, we propose our algorithm, which first makes a state-weighted product automaton from the inputs, over which a lasso with minimum weight should be synthesized. 
We prove that computing a shortest such lasso, even with any constant approximation factor, is computationally hard.  Thus, we introduce an algorithm using a greedy approach to synthesize a short (but not necessarily shortest) lasso with minimum weight.
In Section~\ref{sec:case}, we present two case studies, and finally, in Section~\ref{sec:conc}, we present concluding remarks and discuss future work.

\section{Related Work}
\label{sec:rel} 

Our temporal logic planning is related to, but distinct from, several threads of prior work, which consider temporal logic planning in situations where no plan satisfying a given temporal logic formula can be synthesized.
Fainekos~\cite{fainekos2011revising} introduced an LTL revision problem, 
which, upon failures to plan a trajectory for an LTL formula, provides information about why that failure occurred, and how the LTL formula can be revised so that the transition system has a satisfiable trajectory for the revised formula.
%
%Kim \emph{et al.}~\cite{kim2012revision,kim2015minimal} consider how to revise B{\"u}chi automata representing the LTL specifications such that the system has a trajectory for the revised automaton.
%
%They propose metrics to measure distance between B{\"u}chi automata, prove that finding the ``closest'' automaton is NP-hard, and then provide a SAT-based encoding and a heuristic algorithm for finding the ``closest'' B{\"u}chi automaton for which the transition system has a trajectory.
%
\changed{
Kim \emph{et al.}~\cite{kim2012revision,kim2015minimal} consider the \emph{minimal revision problem} (MRP), which aims to find for a given specification (B{\"u}chi) automaton, a ``closest'' specification automaton for which the system has a trajectory.
They prove that MRP is NP-hard, and then provide a SAT-based encoding and a heuristic algorithm for solving MRP.}
%
%Lahijanian \emph{et al.}~\cite{lahijanian2015time} consider the case where the specifications are given as a \emph{syntactically co-safe LTL}~\cite{kupferman2001model} formula---a specific kind of LTL formula that is satisfied in a finite (rather an infinite) time.
%
\changed{
Lahijanian \emph{et al.}~\cite{lahijanian2015time} propose, based on a user-defined priority over atomic propositions, which they assumed to be low level tasks, an approach to measure how ``close'' is a trajectory to satisfy a given formula, and accordingly, propose an algorithm that generates a trajectory that has the minimum distance to the satisfaction of that given formula.
}%
Lahijanian and Kwiatkowska~\cite{lahijanian2016specification} later extended that idea for probabilistic environments modeled by MDPs.

Two recent results by Dimitrova \emph{et al.}~\cite{dimitrova2018maximum} and Tomita \emph{et al.}~\cite{tomita2017safraless} consider a problem, called \emph{maximum realizability problem}, which is a synthesis problem from a hard constraint and a set of soft constraints in the form of LTL formulas.
%
%This problem is different than our problem in that the aim of our problem is to synthesize a trajectory in a transition system while the aim of maximum realizability problem is to synthesize a \emph{reactive transition system}.
\changed{The aim of this problem is to synthesize a \emph{reactive transition system} (rather than a trajectory within a transition system).}
%This problem is different than our problem in that the aim of our problem is to synthesize a trajectory in a transition system while the aim of maximum realizability problem is to synthesize a \emph{reactive transition system}.
 %
 They consider the case where the soft constraints are of a specific kind of LTL formulas, those who assert that something \emph{globally} holds. 
 Their ideas are based on optimally refining or relaxing the soft specifications such that the resulting soft specifications along with the hard specification are realizable by a reactive transition system.
%

% Smith \emph{et al.}~\cite{smith2011optimal} consider a temporal logic path planing problem in which the goal is to generate for an LTL formula, a trajectory that infinitely often visits specific regions of interest such that the maximum time delay between visiting instances of those regions is minimized.
	%

The closest work to ours is by Tumova \emph{et al.}~\cite{tuumova2013minimum}, who address a similar problem but without the hard constraint; they consider the problem of making a trajectory that maximizes the sum of rewards from satisfying a set of conflicting LTL formulas. 
%
%Their algorithm first makes a \emph{generalized B{\"u}chi automaton} for each LTL formula, and then from those automata, using the idea of converting a generalized B{\"u}chi automaton to a B{\"u}chi automaton~\cite{baier2008principles}, it makes a transition-weighted B{\"u}chi automaton, which has a layer for each of the original generalized B{\"u}chi automata to keep track of the set of LTL formulas for which a run over the automaton is satisfying.
\changed{Their algorithm first makes a \emph{generalized B{\"u}chi automaton} for each LTL formula, and then from those automata, using the idea of converting a generalized B{\"u}chi automaton to a B{\"u}chi automaton~\cite{baier2008principles}, it makes a transition-weighted B{\"u}chi automaton, in which the Cartesian product of the state spaces of the automata are copied into different layers, a layer for each of the original automata, to keep track of the set of LTL formulas for which a run over the automaton is satisfying.}
From this transition-weighted B{\"u}chi automaton and the transition system, a product automaton is constructed, on which an accepting lasso is synthesized using a modified version of nested-DFS~\cite{courcoubetis1992memory}.
\changed{Our algorithm, which is simpler, constructs a state-weighted (rather than a transition-weighted) product automaton, and then uses a greedy approach to 
synthesize on this product automaton, a short accepting lasso with minimum weight; we prove that an accepting lasso with minimum length and weight is computationally hard to find.}
%
%We also prove that it is computationally hard to compute the shortest accepting lasso, and thus, our algorithm uses a greedy approach to synthesize a shortest accepting lasso.
%

Our synthesis process over this product automaton is performed in two passes, first is the one-pass DFS of Tarjan's algorithm~\cite{tarjan1972depth} to compute the set of strongly connected components (SCCs) of the product automaton, and second is a pass that synthesizes the prefix of the lasso as a simple path from the initial state to a leader of a SCC with minimum weight and the suffix of the lasso as a cycle within that SCC using BFS iteratively.
The work of Tumova \emph{et al.}~\cite{tuumova2013minimum} does not consider synthesis of a shortest accepting lasso.
Two other results from the same authors~\cite{tumova2013least, castro2013incremental}
%by Tumova \emph{et al.}~\cite{tumova2013least} and Castro \emph{et al.}~\cite{castro2013incremental}, 
consider for the classical setting of path planning, generating a finite trajectory that minimizes the amount of time the robot deviates only the less important ones of a set of conflicting safety rules.
This problem is for finite trajectories and is treated differently.    

\section{Definitions and  problem statement}
\label{sec:def}
In this section, we review some preliminary tools and introduce the main problem we address.

\subsection{Preliminaries}
The set of infinite-length words over an alphabet $\Sigma$ is denoted $\Sigma ^ \omega$, and the infinite repetition of a finite word $r \in \Sigma^+$ is denoted $r^\omega$.
%
%Accordingly, for a regular expression $\mathrm{e}$ over $\Sigma$, $e^\omega$ is a \emph{$\omega-$regular} expression, identifying the set of infinite words formed by infinite repetition of some word specified by $e$.
%
%Similarly, the infinite repetition of a finite word $r \in \Sigma^+$ is denoted $r^\omega$.
%Given , the infinite repetition of $r$ is denoted $r^\omega$.
%
A \emph{lasso} is formed when such an infinite repetition is concatenated to a finite word, that is, a lasso is an infinite word of the form $r_1(r_2)^\omega$, in which $r_1 \in \Sigma^*$ and $r_2 \in \Sigma^+$.

%\subsection{Transition system}
%\label{subsec:trans}
The environment is modeled as a transition system.
\begin{definition}[transition system]
A \emph{transition system} $\mathcal{T}=(S, R, s_0, AP, L)$ consists of
    a finite set of states $S$; 
    a transition relation $R \subseteq S \times S$; 
    an initial state $s_0 \in S$ ;
    a set of atomic propositions $AP$; 
    and a labeling function $L:S \rightarrow 2^{AP}$, which assigns to each state, a set of atomic propositions, which are properties that hold at that state.
\end{definition}
An execution of the system goes through an \emph{infinite path} $\pi = s_0s_1s_2\cdots \in S^{\omega}$, in which $s_0$ is the initial state and for each $i \geq 0$, $(s_i, s_{i+1}) \in R$. 
The \emph{trace} of this path is $\trace(\pi)=L(s_0)L(s_1)L(s_2) \cdots \in (2^{AP})^\omega$.
The transition systems we deal with should be free of \emph{blocking states}---those states that do not have outgoing edges.
To specify a set of traces, one can use a variety of logical formulas, including those in LTL.
\begin{definition}[LTL syntax]
  \label{def:LTLSyntax}
  \emph{An LTL formula} is generated over a set of atomic propositions $AP$ by the following grammar
  $$\varphi ::= \top \mid p \mid \neg \varphi \mid \varphi \vee \varphi \mid \bigcirc\varphi \mid \varphi \:\mathcal{U} \varphi,$$
  in which $p \in AP$, $\top$ represents the constant $true$, and $\bigcirc$ ('next') and $\:\mathcal{U}$ ('until') are temporal operators.
\end{definition}
An LTL formula $\varphi$ specifies a set of infinite words over $2^{AP}$, denoted $\Words(\varphi)$, which consists of those words who \emph{satisfy} $\varphi$.
To see if a word (trace) $\sigma = A_0 A_1 A_2 \cdots \in (2^{AP})^\omega$ \emph{satisfies} an LTL formula $\varphi$, denoted $\sigma \vDash
  \varphi$, one can use rules that (1) $\varphi \vDash \top$ (2) $\sigma \vDash p$  iff $p \in A_0$ (3) $\sigma \vDash \neg \varphi$  iff $\sigma \nvDash \varphi$ (4) $\sigma \vDash \varphi_1 \vee \varphi_2$  iff $\sigma \vDash \varphi_1$ or $\sigma \vDash \varphi_2$ (5) $\sigma \vDash \bigcirc\varphi$  iff $\sigma[1..] \vDash \varphi$ (6) $\sigma \vDash \varphi_1\mathcal{U}\varphi_2$  iff $\exists j \geq 0, \sigma[j..] \vDash \varphi_2$ and $\forall 0 \leq i < j, \sigma[i..] \vDash \varphi_1$.
For simplicity, two other temporal operators $\lozenge$ ('eventually'), defined as $\lozenge \varphi := \top \:\mathcal{U} \varphi$, and $\square$ ('globally'), defined as $\square \varphi := \neg \lozenge \neg \varphi$, are also used.
%
%Temporal operator $\lozenge$ ('eventually') is also defined as $\lozenge \varphi := \top \:\mathcal{U} \varphi$, and the dual of it $\square$ ('globally') is defined as $\square \varphi := \neg \lozenge \neg \varphi$.
%
We can also use the dual of $\top$, which is $\bot$, the dual of the Boolean operator $\vee$, which is $\wedge$, as well as the Boolean operator $\rightarrow$.
 %

% 
% For any formula $\varphi$, $\Words(\square \varphi) = \Words(\neg \lozenge \neg \varphi)$, or in other words, $\sigma \vDash \square \varphi$ if $\sigma[i..] \vDash \varphi$ for all $i \geq 0$.
%
Each LTL formula is equivalent to a certain type of automaton.
\begin{definition}
  A \emph{B{\"u}chi automaton} $\mathcal{A}=(Q, \Sigma, \delta, q_0,
  F)$ consists of a finite set of states $Q$; an alphabet $\Sigma$; a transition relation $\delta
  \subseteq Q \times \Sigma \times Q$; an initial state $q_0 \in Q$ ; and a set of accepting (final) states $F \subseteq Q$. 
\end{definition}
A run over the automaton is an infinite sequence $r=q_0q_1q_2\cdots \in Q^\omega$ in which $q_0$ is the initial state and for each $i \geq 0$, $(q_i, A, q_{i+1}) \in \delta$ for an $A \in \Sigma$. 
Sequence $r$ is a run for an infinite word $A_0 A_1 A_2 \cdots \in \Sigma^\omega$ if for each integer $i \geq 0$, $(q_i, A_i, q_{i+1}) \in \delta$. 
The set of states that appear infinitely many times
in an infinite run $r$ is denoted $\inf(r)$.
Accordingly, run $r$ is \emph{accepting} if $\inf(r) \cap F \neq \emptyset$.
Consequently, the language of $\mathcal{A}$, denoted $\mathcal{L}_\omega(\mathcal{A})$ is
$$\mathcal{L}_{\omega}(\mathcal{A}) = \lbrace w \in \Sigma^\omega \mid
\text{there exists an accepting run for } w \rbrace.$$
A connection between LTL formulas and B\"uchi automata is that for any LTL formula $\varphi$ over a set of atomic propositions $AP$, one can construct a B\"uchi automaton $\mathcal{A}_\varphi$ with alphabet $2^{AP}$ such that $\mathcal{L}_\omega(\mathcal{A}_\varphi) = \Words(\varphi)$. %
Several algorithms for this kind of construction are available~\cite{vardi1994reasoning, somenzi2000efficient, gastin2001fast,babiak2012ltl}.

In our algorithm, we need to be sure that each B\"uchi automaton created for a soft constraint is $\emph{nonblocking}$, that is, for any state $q$ in the automaton and every letter $a \in \Sigma$, there is at least one state $q^\prime$ such that $(q, a, q^\prime) \in \delta$. 
Any B\"uchi automaton is converted to a nonblocking B\"uchi automaton by adding a trapping state, to which all missing transitions are added.
We also consider a variant of B\"uchi automaton called \emph{generalized B\"uchi automaton}, which has the same syntax of B\"uchi automaton except that it has a set $\mathcal{F} \subseteq 2^Q$ rather than a set $F \subseteq Q$ as its acceptance set. 
More precisely, the acceptance set of the automaton is a set $\mathcal{F}$ consisting of sets $F_1, F_2, \ldots F_k$ with $F_i \subseteq Q$ for each $i \in \lbrace 1,. \cdots, k \rbrace$.
Accordingly, an infinite run $r$ over a generalized B\"uchi automaton $\mathcal{G}$ is accepting if for each $F \in \mathcal{F}$, it holds that $\inf(r) \cap F \neq \emptyset$.
The language of $\mathcal{G}$, $L_{\omega}(\mathcal{G})$, is the set of all infinite words for each of which there is an accepting run.
%Consequently, the \emph{language} $\mathcal{L}_{\omega}(\mathcal{G})$, of a generalized B{\"u}chi
%automaton $\mathcal{G}$, is 
%$$\mathcal{L}_{\omega}(\mathcal{G}) = \lbrace w \in \Sigma^\omega \mid
%\text{there exists an accepting run } r \text{ for } w \rbrace.$$
%Having reviewed these standard definitions, we are ready to state our problem.

\subsection{LTL planning with soft constraints}
Our goal in this problem is to find, in a transition system modeling the environment, an infinite path whose trace satisfies a goal mission expressed as an LTL formula $\varphi$ while optimally satisfying a prioritized list of soft constraints $\psi_1, \psi_2, \cdots, \psi_n$, where each $\psi_i$ is an LTL formula, given in order of decreasing importance.
For this purpose, we define a cost function $f_{\omega}: (2^{AP})^\omega \rightarrow \mathbb{Z}_{\geq 0}$, such that for any $\sigma \in (2^{AP})^\omega$,
\begin{equation}\label{eq:cost-inf}
  f_\omega(\sigma) = \sum\limits_{i:\sigma \notin \Words(\psi_i)} n^{n-i}.
\end{equation}
Note that this cost function guarantees to impose the standard lexicographic ordering between all Boolean vectors, where each vector has an entry for each LTL constraint showing whether that LTL constraint is satisfied or not.
\added{As a result, a constraint with a higher priority (smaller number) is never sacrificed to satisfy a constraint with a lower priority.}
Accordingly, we want a trajectory whose trace minimizes this function.
With this in mind, our problem is defined as:
\problem{Optimal LTL Planning with Soft Constraints (OLPSC)}
  {A transition system $\mathcal{T}$, an LTL formula $\varphi$, and a prioritized list of $n$ LTL formulas $\psi_1, \psi_2, \cdots, \psi_n$.}
  {An infinite path $\pi$ over $\mathcal{T}$ such that $\trace(\pi) \vDash \varphi$ and $f_\omega(\trace(\pi))$ is minimized.}

%The next section, proposes an algorithm for solving this problem.

\section{Algorithm Description}
\label{sec:alg}
This section presents an algorithm for solving the OLPSC problem. 
See Algorithm~\ref{alg:opp}.
The two main steps of this algorithm are constructing a product automaton (line~\ref{line:Prod}), and computing a lasso with minimum cost on the product automaton (line~\ref{line:OptLasso}).
In the sequel, we explain those steps.
%
%

%The first step, Lines~\ref{line:ForLoopPsis}-~\ref{line:Phi}, constructs B\"uchi automata representations of the LTL formulas. It ensure that those who represent LTL preferences are nonblocking. 
%
%The next step, Line~\ref{line:Prod}, constructs from these B\"uchi automata, a product automaton.
%
%An optimal accepting lasso over this product automaton is synthesized, Line~\ref{line:OptLasso}, and is projected to the transition system, Line~\ref{line:lasso2path}, to obtain an optimal path, which is returned as the output.
%

\subsection{The product automaton}\label{sec:prodc}

The first step of the algorithm is, following an established pattern in the literature~\cite{tuumova2013minimum, lahijanian2015time, tumova2013least, castro2013incremental}, to construct a form of product automaton~\cite{vardi1986automata}.
%is a well-known technique in the literature, and depending on their problem, each related work~ constructs their own kind of product automaton; we also make our own product automaton.
%
To that end, the algorithm first %(lines~\ref{line:ForLoopPsis}--\ref{line:Phi})
makes the B\"uchi automata representations of the LTL formulas ---an automaton $\mathcal{A}$ for  $\varphi$, and an automaton $\mathcal{B}_i$ for each $\psi_i$.
It then %(line~\ref{line:nonblocking})
ensures that $\mathcal{B}_i$'s are nonblocking and uses all those automata along with the transition system to construct
%(ine~\ref{line:Prod})
a product automaton based on the following definition. 
\begin{definition}
  \label{def:prodAuto}
  For a B{\"u}chi automaton $\mathcal{A}=(Q, 2^{AP}, \delta, q_0, F)$, a transition system $\mathcal{T}=(S, R, s_0, AP, L)$, and a prioritized list of
  $n$ nonblocking B\"uchi automata $\mathcal{B}_i=(Q_i, 2^{AP}, \delta_i, q_{0, i}, F_i)$ for $i \in \lbrace 1, \ldots, n \rbrace$,
  the \emph{product automaton}
  %$\mathcal{P}=\mathcal{A} \times \mathcal{T} \times \mathcal{B}_1 \times \ldots \times \mathcal{B}_n$ is a tuple
  is a tuple $\mathcal{P}=(Q_\mathcal{P}, \delta_\mathcal{P}, q_{0,\mathcal{P}},
  F_\mathcal{P}, \mathbf{w})$ in which
  \begin{enumerate}
      \item $Q_\mathcal{P} = Q \times S \times Q_1 \times \ldots \times Q_n$ is a finite set of states;
      \item $q_{0, \mathcal{P}}=(q_0, s_0, q_{0, 1}, \ldots, q_{0, n})$ is the initial state;
      \item $\delta_{\mathcal{P}} \subseteq Q_{\mathcal{P}} \times
      Q_\mathcal{P}$ is a transition relation, such that $((q, s, q_1, \ldots, q_n),
      (q^\prime, s^\prime, q^\prime_1, \ldots, q^\prime_n)) \in \delta_P$ if and only if $(s,
      s^\prime) \in R$, $(q, L(s), q^{\prime}) \in \delta$, and $(q_i, L(s), q^{\prime}_i) \in \delta_i$ for each $i \in \lbrace 1, \ldots, n \rbrace$;
      \item $F_\mathcal{P} \!=\! F \times S \times Q_1 \times \ldots \times Q_n$ is the set of accepting
      states;
      \item $\mathbf{w}: Q_\mathcal{P} \rightarrow \lbrace \T, \F \rbrace^n$ is a state-weighting function that assigns to each state $(q, s, q_1, \ldots, q_n) \in
      Q_\mathcal{P}$, a Boolean vector $\mathbf{v}$ such that for any $1 \leq i \leq n$, it holds that $\mathbf{v}[i]=\T$ if and only if $q_i \in F_i$.
  \end{enumerate}
\end{definition}

This product automaton can be thought of as a B\"uchi automaton with a trivial alphabet, and thus, all definitions related to B\"uchi automata are applicable on it. 

For a state $(q, s, q_1, \ldots, q_n) \in
      Q_\mathcal{P}$, $\mathbf{w}((q, s, q_1, \ldots, q_n))$ indicates which of the $q_i$'s were accepting in their original B\"uchi automata.
Accordingly, for any $1 \leq i \leq n$, we use $\mathrm{F}_{i, \mathcal{P}}$ to denote in $\mathcal{P}$, the set of all states that are accepting for automaton $\mathcal{B}_i$, i.e., 
$\mathrm{F}_{i, \mathcal{P}} = \lbrace p \in Q_{\mathcal{P}} \mid \mathbf{w}(p)[i]=\mathrm{T} \rbrace$.
For a run $r_\mathcal{P}=q_0 q_1 q_2 \ldots \in Q_{\mathcal{P}}^\omega$, we use $\mathbf{inf}(r_\mathcal{P})$ to denote a vector $\mathbf{v} \in \lbrace \T, \F \rbrace ^n$ in which for each $1 \leq i \leq n$, $\mathbf{v}[i] = \T$ if and only if there are infinitely many $j \geq 0$ such that $\mathbf{w}(q_j)[i] = \T$.
Subsequently, by having a cost function $f_{\mathbf{w}}: \lbrace \T, \F \rbrace ^ n \rightarrow \mathbb{Z}_{\geq 0}$, in which for any $\mathbf{v} \in \lbrace \T, \F \rbrace ^ n$, 
\begin{equation}
    f_{\mathbf{w}} (\mathbf{v}) = \sum\limits_{i:\mathbf{v}[i] = \F} n^{n-i},
\end{equation}
the cost of $r_\mathcal{P}$ will be $f_{\mathbf{w}}(\mathbf{inf}(r_\mathcal{P}))$. 
The purpose of constructing $\mathcal{P}$ is to synthesize a run $r_{\mathcal{P}}$ that has the minimum cost.
To see why, we first consider the following lemmas.
\begin{lemma}
  \label{lem:Prod2TS}
  Given the structures in Definition~\ref{def:prodAuto}, let $r_\mathcal{P}=(q_0, s_0, q_{0, 1}, \ldots, q_{0, n})(q_1, s_1, q_{1, 1}, \ldots, q_{1, n})\cdots$ be a run over $\mathcal{P}$. It holds that:
  \begin{enumerate}
      \item If $r_\mathcal{P}$ is accepting for $\mathcal{P}$, then the sequence $\pi=s_0s_1s_2\cdots$ is a path
  for $\mathcal{T}$ such that $\trace(\pi) \in \mathcal{L}_\omega(\mathcal{A})$.
  
      \item For any $i \in \lbrace 1, \ldots, n \rbrace$, if $\mathbf{inf}(r_{\mathcal{P}})[i] = \T$, then the sequence $\pi=s_0s_1s_2\cdots$ is a path
  for $\mathcal{T}$ such that $\trace(\pi) \in \mathcal{L}_\omega(\mathcal{B}_i)$.

  \end{enumerate}
\end{lemma}

\begin{algorithm}[t]
  \caption{\textsc{OptimalLTLPlanningWSoftConsts}} 
  \label{alg:opp}
  %\Require $\mathcal{T}$, $\varphi$, $\psi_1, \psi_2, \ldots, \psi_n$
  %\hspace*{\algorithmicindent} \textbf{Input:} $\mathcal{T}$, $\varphi$, $\psi_1, \psi_2, \ldots, \psi_n$ \\
  %\hspace*{\algorithmicindent} \textbf{Output:} A path $\pi=s_0s_1s_2\cdots$ on $\mathcal{T}$ s.t $\pi \vDash \varphi$ and $f_{\omega}(trace(\pi))$ is minimal.
  \KwData{$\mathcal{T}$, $\varphi$, $\psi_1, \psi_2, \ldots, \psi_n$}
  \KwResult{A path $\pi=s_0s_1s_2\cdots$ on $\mathcal{T}$ s.t $\pi \vDash \varphi$ and $f_{\omega}(trace(\pi))$ is minimum}
 %\KwInput{salam}
 %\SetKwInOut{Input}{Input}
 %\SetKwInOut{Output}{Ouput}
 %\REQUIRE $n \geq 0 \vee x \neq 0$
 %\ENSURE $y = x^n$
  
  \DontPrintSemicolon
  \BlankLine
  %\SetKwInOut{Input}{Input}
  %\SetKwInOut{Output}{Output}
  
  \For {$i = 1$ to $n$}{\label{line:ForLoopPsis}
    %{$DFAList \gets \textsc{ADD2List}(DFAList, \textsc{LDL2DFA}(\psi_i))$} \;
    {$\mathcal{B}_i \gets \textsc{LTL2B\"uchiAutomaton}(\psi_i)$} \label{line:B1} \;
    {$\mathcal{B}_i \gets \textsc{MakeNonblocking}(\mathcal{B}_i)$} \label{line:nonblocking}  \;
  }
  {$\mathcal{A} \gets \textsc{LTL2B\"uchiAutomaton}(\varphi)$}\label{line:Phi}\;
  
  {$\mathcal{P} \gets \mathcal{A} \times \mathcal{T} \times \mathcal{B}_1 \times \mathcal{B}_2 \cdots \mathcal{B}_n $}\label{line:Prod}\;

  {$r \gets \textsc{MinimumCostAcceptingLasso}(\mathcal{P})$}\label{line:OptLasso} \;
  
  {\textbf{if} $r = \textbf{nil}$ \textbf{then} \Return{\rm nil}}
  
  %\If {$r = \textbf{nil}$}{ \label{line:E} \Return{\rm nil} }

%  {$\pi = \textsc{Convert2PathOnTS}(r)$} \label{line:lasso2path}\;
%  \Return{$\pi$}
{\Return{$\textsc{Convert2PathOnTS}(r)$}} \label{line:lasso2path}\;
  
\end{algorithm} 

\begin{proof} 
(1) From the construction of $\mathcal{P}$, it follows that the sequence $\pi=s_0s_1s_2\cdots$ ---the projection of $r_{\mathcal{P}}$ onto $\mathcal{T}$--- is a path over
$\mathcal{T}$, and that the sequence $r=q_0q_1q_2\cdots$ is a run for
$\trace(\pi)=L(s_0)L(s_1)L(s_2)\cdots$ over $\mathcal{A}$.
Given that $r_{\mathcal{P}}$ is an accepting run for $\mathcal{P}$, there are infinitely many $i$'s for $r=q_0q_1q_2\cdots$ such that $q_i \in F$, implying that $r$ is accepting, and thus, $\trace(\pi) \in \mathcal{L}_{\omega}(\mathcal{A})$.
(2) The proof is similar to the proof of (1) with the consideration that in this case, for each $i$, sequence $r_i=q_{0, i} q_{1, i} q_{2, i} \cdots$ is a run for
$\trace(\pi)=L(s_0)L(s_1)L(s_2)\cdots$ over $\mathcal{B}_i$. 
\end{proof}

\begin{lemma}
  \label{lem:TS2PROD}
  Assuming the structures in Definition~\ref{def:prodAuto}, for any $I \subseteq \lbrace 1, \ldots, n \rbrace$ and for any path $\pi=s_0s_1s_2\cdots$ in $\mathcal{T}$ such that $\trace(\pi) \in \mathcal{L}_\omega(\mathcal{A})$ and
  $\trace(\pi) \in \bigcap_{i \in I} \mathcal{L}_\omega(\mathcal{B}_i)$, there exists an accepting
  run $r_\mathcal{P}=(q_0, s_0, q_{0, 1}, \ldots, q_{0, n})(q_1, s_1, q_{1, 1}, \ldots, q_{1, n})\cdots$ over
  $\mathcal{P}$ such that $\mathbf{inf}(r_\mathcal{P})[i] = \T$ for all $i \in I$.
\end{lemma}

\begin{proof}
Let $r$ be an accepting run for $\pi$ over $\mathcal{A}$, and let for each $i \in I$, $r_i$ be an accepting run for $\pi$ over $\mathcal{B}_i$.
Given that all B\"uchi automata created for the soft constraints are nonblocking, for each $j \in \lbrace 1, 2, \cdots n \rbrace $ such that $j \notin I$, there exists an infinite run $r_j$ for $\pi$ over $\mathcal{B}_j$. Now we choose one such $r$, one such $r_i$ for each $i$, and one such $r_j$ for each $j$.  Then we combine $\pi$, the chosen $r$, all the chosen $r_i$'s, and all the chosen $r_j$'s to form an $r_\mathcal{P}$. This constructed $r_\mathcal{P}$ has the properties claimed in this lemma.
\end{proof}

The impact of these lemmas is that for any optimal solution of the OLPSC problem, there is an accepting run $r_\mathcal{P}$ over $\mathcal{P}$ for which $f_\mathbf{w}(\mathbf{inf}(r_{\mathcal{P}}))$ is minimum, and that from any accepting run $r_\mathcal{P}$ that minimizes $f_\mathbf{w}(\mathbf{inf}(r_{\mathcal{P}}))$, one can create an optimal solution to the OLPSC problem via projecting $r_\mathcal{P}$ into $\mathcal{T}$.
Accordingly, one can solve the OLPSC problem by computing over $\mathcal{P}$, a run $r_{\mathcal{P}}$ with minimum cost.
\subsection{Trajectory generation}
A run $r_\mathcal{P}$ with minimum cost is constructed in Line~\ref{line:OptLasso} of Algorithm~\ref{alg:opp}.
The product automaton may have many, or even infinitely many optimal runs; \added{in fact, there could exist an optimal run whose sequence of states cannot be specified by any pattern}; however, we are interested in only one kind, which is revealed by the following result.
\begin{lemma}
  \label{lem:lasso}
  If  $\mathcal{L}_{\omega}(\mathcal{P}) \neq \emptyset$, then $\mathcal{P}$ has an accepting lasso $r_{\mathcal{P}}=r_1(r_2)^\omega$ such that $r_1 \in Q_{\mathcal{P}}^*$, $r_2 \in Q_{\mathcal{P}}^+$, and that $f_{\mathbf{w}}(\mathbf{inf}(r_{\mathcal{P}}))$ is minimum.
\end{lemma}

\begin{proof}
  We show that from any accepting run $r_{\mathcal{P}}^\prime = p_0 p_1 p_2 \ldots$ that minimizes $f_{\mathbf{w}}(\mathbf{inf}(r_{\mathcal{P}}^\prime))$, we can construct an accepting lasso $r_{\mathcal{P}}$ such that $f_{\mathbf{w}}(\mathbf{inf}(r_{\mathcal{P}}^\prime)) = f_{\mathbf{w}}(\mathbf{inf}(r_{\mathcal{P}}))$.
  Given that $r_{\mathcal{P}}^\prime$ has an infinite length while $Q_{\mathcal{P}}$ has only a finite number of states, there exists an integer $k \geq 0$ such 
  %that all states $p_j$ in $r_{\mathcal{P}}^\prime$ with $j \geq k$ are repeated infinitely many times in $r_{\mathcal{P}}^\prime$, i.e., 
  $\inf(r_{\mathcal{P}}^\prime) = \lbrace p_j \in r_{\mathcal{P}}^\prime \mid j \geq k \rbrace$.
  We choose $l$ to be the smallest such $k$.

  Let $I = \lbrace 1 \leq i \leq n \mid \mathbf{inf}(r_{\mathcal{P}}^\prime)[i]=\T \rbrace$.
  We choose an integer $j \geq l$ such that $r_{\mathcal{P}}^\prime[l..j]$ contains at least one state $p \in F_{\mathcal{P}}$ and it contains at least a state $q_i \in F_{i, \mathcal{P}}$ for each $i \in I$.
  %Given that $r_{\mathcal{P}}^\prime$ is accepting, there exists an accepting state $p \in Q_{\mathcal{P}}$ such that $p \in \inf(r_{\mathcal{P}}^\prime[t..])$. In addition, for each $i %\in I$, there is a state $ \mathrm{F}_{i, \mathcal{P}}$ such that $q_i \in r_{\mathcal{P}}^\prime[t..]$. We choose one such $p$, and one such $q_i$ for each $i \in I$.
  %
  %Clearly, there exist an integer $j \geq t$ such that $p$ has appeared in $r_{\mathcal{P}}^\prime[t..j]$ at least once and that for each 
 %$i \in I$, $q_i$ has appeared in $r_{\mathcal{P}}^\prime[t..j]$.
 %
 We choose $u$ to be the smallest such integer $j$.
 We now set $r_2=r_{\mathcal{P}}^\prime[l..u]=p_l p_{l+1} \ldots p_{u}$ and set $r_1=r_{\mathcal{P}}^\prime[0..l-1]=p_0 p_{1} \ldots p_{l-1}$.
 Clearly, lasso $r_\mathcal{P}$ is accepting.  Moreover, $f_{\mathbf{w}}(\mathbf{inf}(r_{\mathcal{P}}^\prime)) = f_{\mathbf{w}}(\mathbf{inf}(r_{\mathcal{P}}))$.
\end{proof}

\begin{algorithm}[t]
  \caption{MinimumCostAcceptingLasso}
  \label{alg:optLasso}
  \KwData{Product automaton $\mathcal{P}=(Q_{\mathcal{P}}, \delta_{\mathcal{P}}, q_{0, \mathcal{P}}, F_{\mathcal{P}}, \mathbf{w})$}
  \KwResult{An accepting lasso for $\mathcal{P}$ minimizing $f_{\mathbf{w}}$ }
  {$SCCs = \textsc{StronglyConnectedComponents}(\mathcal{P})$}\; \label{line:CompSCCS}
  {$O \gets \textbf{nil}$}\; \label{line:defOptC} 
  {$minW \gets \infty $}\;
  \ForAll{$C \in SCCs$}{ \label{line:loopOptSCCStart}
    \If{$C.accepting = \textbf{True}$}{
%         {$\textbf{continue}$}\;
         \If{$f_\mathbf{w}(C.\textbf{w}) < minW$}{
        {$O \gets C$}\;
        {$minW \gets f_\mathbf{w}(C.\textbf{w})$}\;
    }
    }
  }  \label{line:loopOptSCCEnd}
  
  %\BlankLine
  {\textbf{if} $O = \textbf{nil}$ \textbf{then} \textbf{return nil}}
 % \If{$optC = \textbf{nil}$}{
  %   {\textbf{return nil}}\;
  %}
  
 % \BlankLine  
  {$r_1 = \textsc{BFSShortestPath}(q_{0, \mathcal{P}}, O.leader)$}\;\label{line:r1}
  {$r_2 = \textsc{MinCostAcceptingCycle}(O)$}\; \label{line:r2}
  {\textbf{return} $(r_1, r_2)$}\;

\end{algorithm}

 The punchline is that to synthesize an optimal run, it is sufficient to consider those runs who are lassos.
 We are also interested in finding a shortest such lasso ---a lasso $r_\mathcal{P} = r_1 (r_2)^w$ for which $|r_1|+|r_2|$ is minimum.
 Unfortunately, the following result reveals that finding a shortest such lasso is not easy.
 \begin{lemma}
  \label{lem:shortestLasso}
  Given a product automaton $\mathcal{P}$, the problem of finding over $\mathcal{P}$, a shortest lasso $r_\mathcal{P}=r_1(r_2)^\omega$ that minimizes $f_\omega(r_\mathcal{P})$ is NP-hard. 
\end{lemma}

\begin{proof}
We prove by reduction from the problem of finding a shortest accepting lasso for a generalized B\"uchi automaton, which is known to be NP-hard~\cite{clarke1995efficient, ehlers2010short}.
For each generalized B\"uchi automaton $\mathcal{G}=(Q, \Sigma, \delta, q_0, \mathcal{F}:=\lbrace F_1, F_2, \cdots, F_n \rbrace )$, we make a product automaton $\mathcal{P}=(Q_\mathcal{P}, \delta_\mathcal{P}, q_{0, \mathcal{P}}, F_\mathcal{P}, \mathbf{w})$ such that $Q_\mathcal{P} = Q$; $q_{0, \mathcal{P}} = q_0$; $F_\mathcal{P} = Q$; for each $q, q^\prime \in Q$, $(q, q^\prime) \in \delta_\mathcal{P}$ iff $(q, a, q^\prime) \in \delta$ for an $a \in \Sigma$; and for each state $q \in Q_{\mathcal{P}}$, function $\mathbf{w}$ assigns a vector $\mathbf{v} \in \lbrace \T, \F \rbrace^n$ such that for each $j \in \lbrace 1, 2, \cdots, n \rbrace$, $\mathbf{v}[j]=\T$ if $q \in F_j$, and otherwise, $\mathbf{v}[j]=\F$.
Consider that any run over $\mathcal{P}$ is accepting, and that $f_{\mathbf{w}}(\mathbf{inf}(r_\mathcal{P}))=0$ for any optimal lasso $r_\mathcal{P}$ in $\mathcal{P}$.
Any shortest lasso $r_\mathcal{P}$ over $\mathcal{P}$ for which $f_{\mathbf{w}}(\mathbf{inf}(r_\mathcal{P}))=0$ is a shortest accepting lasso for $\mathcal{G}$. This completes the proof.
\end{proof}

As a result of this lemma, unless P=NP we cannot compute in a time polynomial to the size of $\mathcal{P}$, a shortest lasso that is accepting for $\mathcal{P}$ and for which $f_{\mathbf{w}}(r_\mathcal{P})$ is minimum. 
Unfortunately, it is also NP-hard to approximate within any constant factor, the length of such a lasso (the proof would utilize the same reduction in Lemma~\ref{lem:shortestLasso} along with the fact due to Ehlers~\cite{ehlers2010short}, according to which it is NP-hard to approximate within any constant factor the length of a shortest accepting lasso for a generalized B\"uchi automaton).
\changed{Consequently, we utilize a greedy algorithm to find a shortest such lasso which has the minimum cost.}

Our algorithm uses graph algorithms to minimize $|r_1|$ and $|r_2|$ separately. 
Algorithm~\ref{alg:optLasso} shows the process.
Consider that $\mathcal{P}$ can be thought of as a directed graph with vertex set $Q_\mathcal{P}$ and edge set $\delta_{\mathcal{P}}$.
Additionally, all vertices (states) in $r_2$ are in a \emph{strongly connected component} (SCC) of the graph given that they are contained in a cycle, $r_2.r_2[0]$.
%
%Recall that a strongly connected component of a graph is a subgraph in which any pair of states are reachable from each other, and that it is maximal with this property, that is, if any other vertices or edges are added to the subgraph, then in the subgraph there will be a pair of states that will not be reachable from each other.
%
%An SCC is \emph{maximal} if it is not contained in any other SCC.
%
With these in mind, our algorithm first decomposes the graph into its strongly connected components, (Line~\ref{line:CompSCCS}); then finds a SCC that contains $|r_2|$ of a lasso $r_\mathcal{P} = r_1 (r_2)^\omega$ with minimum $f_\mathbf{w}(r_\mathcal{P})$ (Lines~\ref{line:defOptC}--\ref{line:loopOptSCCEnd}); and then construct $r_1$ and $r_2$ (Lines~\ref{line:r1} and Line~\ref{line:r2} respectively).
See Figure~\ref{fig:SCCs}.

To find the set of SCCs of the graph, we use the well-know algorithm of Tarjan~\cite{tarjan1972depth}. 
This algorithm uses depth first search (DFS) to traverse all the vertices (states) of the graph in one pass.
During this traversal, each vertex $p$ is assigned a unique integer $p.number$, which is, in fact, the traversal's step number at which $p$ is reached. 
Each vertex is assigned another integer $p.lowlink$, whose value is set to the smallest index of any node reachable from $p$, including $p$ itself.
During this algorithm, all vertices that are assigned the same value of $lowlink$ will be in the same SCC of the graph, and among those vertices, the one whose number is equal to its $lowlink$ is the leader (representative) of the SCC. 
\begin{figure}[t]
  \centering
  \includegraphics[width=\linewidth]{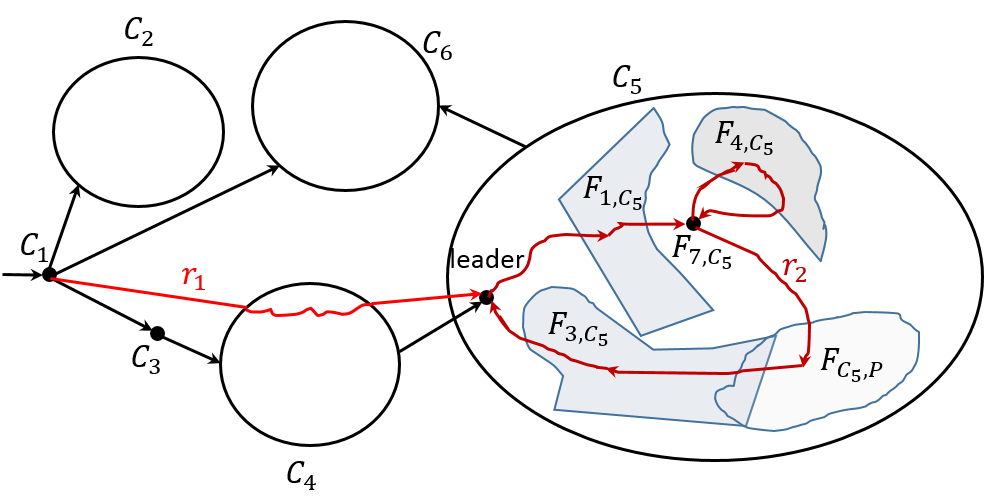}
  
  \caption{
    Showing our algorithm for finding an optimal lasso $r_{\mathcal{P}} = r_1 (r_2)^\omega$ over product automaton $\mathcal{P}$. 
    Each $C_i$ is a strongly connected component of the graph underling $\mathcal{P}$. 
    Component $C_5$ contains the suffix of an optimal lasso. Set $F_{C_5, \mathcal{P}}$ contains those state in $C_5$ that are accepting for $\mathcal{P}$, i.e., $F_{C_5, \mathcal{P}} = F_{\mathcal{P}} \cap C_5$.
    For each $i \in \lbrace 1, 3, 4, 7 \rbrace$, set $F_{i, C_5}$ are those states in $C_5$ that are accepting for B\"uchi automaton $\mathcal{B}_i$---the one who represents preference $\psi_i$.
  }
   \label{fig:SCCs} 
\end{figure}
\vspace*{-8pt}

As Tarjan's algorithm executes, we also compute for each SCC $C$, the value of $C.accepting$, which gets $True$ only when $C$ contains an accepting state of $\mathcal{P}$ and that $C$ is not a singleton $vertex$ who does not have a loop.
We also compute the value of Boolean vector $C.\mathbf{w}$, whose value is set to $C.\mathbf{w} = \sum_{q \in C} \mathbf{w}(q)$.
Notice that during the same pass of the Tarjan's algorithm, one can keep for each vertex, a link to its parent in the DFS search.
Accordingly, later, the algorithm can use those links to find, for each state, a path from the initial state to that state.
These paths can be used in Line~\ref{line:r1} to compute $r_1$, which is, in fact, a path from the initial state to the leader of the component which we choose to construct $r_2$ from.
However, a path $r_1$ that is constructed in that way may not have minimum length.
Accordingly, we use BFS to find a shortest simple path from $q_{0, \mathcal{P}}$ to $C.leader$.

The final phase of finding an accepting lasso is to synthesize the suffix $r_2$ of it, Line~\ref{line:r2} of Algorithm~\ref{alg:optLasso}.
This suffix is synthesized inside an optimal SCC $O$ using the following greedy algorithm.
Let $r_{\mathcal{P}}^\prime \in Q_\mathcal{P}^\omega$ be any run that minimizes $f_\omega(r_\mathcal{P}^\prime)$, and let $I= \lbrace 1 \leq i \leq n \mid \mathbf{w}(r_{\mathcal{P}}^\prime)[i] = \T \rbrace$.
Let for each $i \in I$, $F_{i, O}$ be the set of states in $O$ that are accepting for B\"uchi automaton $\mathcal{B}_i$, i.e., $F_{i, O} = O \cap \lbrace q \in Q \mid \mathbf{w}(q)[i] = \T  \rbrace$, and let $F_{O, \mathcal{P}}$ be the set of accepting states in $O$, i.e., $F_{O, \mathcal{P}} = F_{\mathcal{P}} \cap O$.
Our greedy algorithm synthesizes $r_2$ as the vertices (states) of a cycle starting from $O.leader$ such that the cycle contains at least a state of $F_{O, \mathcal{P}}$ and at least a state in $F_{i, \mathcal{P}}$ for each $i \in I$.
To do so, it uses variable $\mathrm{U}$ with initial value $F_{O, \mathcal{P}} \bigcup \lbrace F_{i, O} | i \in I \rbrace$.
It starts from $O.leader$, and performs \emph{breadth first search} (BFS) until it finds a state $s$ for which there is a set $M \in \mathrm{U}$ such that $s \in M$. Using the parent links set during BFS, the algorithm records the shortest path from $O.leader$ to $s$ as a first portion of $r_2$, and removes from $U$ all those sets $M$ for which $s \in M$. It then does a similar task, BFS traversal, starting from $s$, and then records the path traversed from $s$ to the new found state. It repeatedly does this process until $U$ becomes $\emptyset$. At this time, it does a BFS to find the shortest path back to $O.leader$. By this time, it has made $r_2$ as the vertices of a cycle.
Figure~\ref{fig:SCCs} illustrates how inside SCC $C_5$, the algorithm constructs $r_2$.

Given this discussion, we now analyze the time complexity of our algorithm.
\begin{lemma}
  For any automaton $\mathcal{P}=(Q_\mathcal{P},
  \delta_\mathcal{P}, Q_{0,\mathcal{P}}, F_\mathcal{P}, \mathbf{w})$, a lasso with minimum cost can be generated in time
  $\mathcal{O}(n(|\delta_{\mathcal{P}}|+|Q_{\mathcal{P}}|))$. 
\end{lemma}

\begin{proof}
It takes $\mathcal{O}(|\delta_{\mathcal{P}}|+|Q_{\mathcal{P}}|)$ time to find the set of strongly connected components.  BFS is called at most $n+1$ times, each of which takes $\mathcal{O}(|\delta_{\mathcal{P}}|+|Q_{\mathcal{P}}|)$ time.
\end{proof}

This bound simplifies to $\mathcal{O}(|\delta_{\mathcal{P}}|+|Q_{\mathcal{P}}|)$ if the number of soft constraints $n$ is treated as a constant.

We can slightly improve the quality of solution by letting instead of the leader, the first state to which BFS reaches from the leader and who is either a final state or it corresponds to a final state of the B\"uchi automata for a soft constraint, to be the midpoint of the lasso. 
%Note that we could increase the quality of solution by letting an accepting state of the product automaton to be the midpoint of the lasso rather than the leader of a SCC, and accordingly, we could choose the shortest lasso among those lassos, but that could substantially increase the running time of synthesizing the lasso given that by the construction of Definition~\ref{def:prodAuto}, a product automaton usually has a considerable number of accepting states.

Though Algorithm~\ref{alg:optLasso} generates a lasso of minimum cost, it is not guaranteed to produce the shortest such lasso.
In the next section, we compare our algorithm with an optimal brute-force algorithm, which finds the shortest lasso by letting any state within an optimal SCC to be the midpoint of the lasso, and synthesizes the suffix of the lasso by searching from the shortest ones, all cycles, simple or otherwise, that start from the midpoint until it finds a satisfactory cycle or the length will be longer than the length of a shortest lasso computed for other midpoints.
\section{Case studies and experiments}
\label{sec:case}
We have implemented Algorithms~\ref{alg:opp} and \ref{alg:optLasso} in Java.
The computed results were executed on an Ubuntu 16.04 computer with a 3.6GHz processor.
%\subsection{Case studies}

% In this section, we  present  two example  instances. 

\subsection{Case study: Hospital deliveries}
Figure~\ref{fig:hospital} shows a hospital which has an emergency department ($e$), a primary care department ($p$), a maintenance department ($n$), a pharmacy ($h$), and a warehouse ($w$).
In this hospital, a robot delivers first aid items ($f$) and medicine ($m$) from the warehouse to the other departments. 
\added{Each state of the transition system for this case study has an atomic proposition indicating a unit within the hospital, along with other propositions indicating which items the robot is caring at that unit. Those states are connected according to the connectivity of the units within the hospital and the items the robot can take or leave at a unit.}
The robot's primary tasks are to \emph{deliver first aid items to the emergency department, deliver first aid items to primary care, deliver medicine to the pharmacy, and report for maintenance, each infinitely often}.  In addition, suppose the robot is given these soft specifications, ordered from most to least important:
\begin{enumerate}
        \item If first aid items are delivered to the primary care department, then \changed{do} not deliver additional first aid items there until first aid items \changed{have been delivered} to the emergency department, and vice versa.
    \item If the first aid items are picked from the warehouse, then \changed{they must not be delivered} to the primary care department until the emergency department receives the first aid items. 
%    \item Do not carry first aid items and medicine together.
    \item \changed{Do not carry first aid items and medicine together.}
    \item Always pick both first aid items and medicine when leaving warehouse. 
\end{enumerate}    

Notice that, in particular, the first two constraints cannot be expressed in LDL$_f$. Thus, the algorithm from our prior work~\cite{rahmani2019optimal} cannot generate a plan for this instance. The box below shows how to formulate these constraints into an instance of OLPSC, along with the solution computed by our implementation.

\begin{figure}[t]
  \centering
  \includegraphics[width=\linewidth]{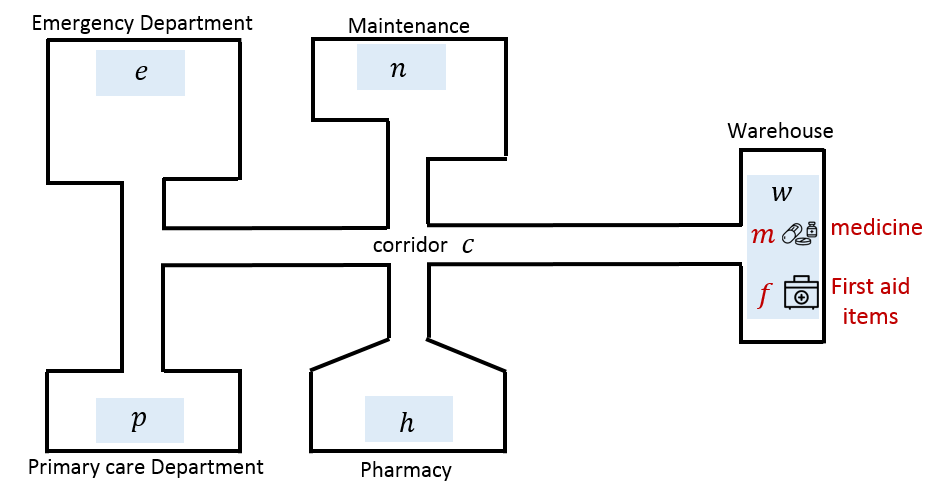}
  
  \caption{A hospital, in which a delivery robot is tasked to deliver first aid items to emergency and primary care departments, deliver medicine to pharmacy, and visit the maintenance section. The robot's task is expressed by LTL formula $\varphi = \square \lozenge (p \wedge f) \wedge \square \lozenge (e \wedge f) \wedge \square \lozenge (h \wedge m) \wedge \square \lozenge n$}
   \label{fig:hospital} 
\end{figure}

\casestudyinternal
{$\square \lozenge (p \wedge f) \wedge \square \lozenge (e \wedge f) \wedge \square \lozenge (h \wedge m) \wedge \square \lozenge n$}
{
\ \ \ 1) $\square ((p \wedge f) \rightarrow \bigcirc (\neg p \:\mathcal{U} (e \wedge f)))$ \\ \ \ \ \ \ \ $\wedge \square ((e \wedge f) \rightarrow \bigcirc (\neg e \:\mathcal{U} (p \wedge f)))$ \\
\ \ \ 2) $\square ((c \wedge f) \rightarrow (\neg p \:\mathcal{U} e))$   \\
\ \ \ 3) $\square (\neg w \rightarrow \neg (f \wedge m))$  \\
\ \ \ 4)    $\square ((w \wedge \bigcirc c ) \rightarrow \bigcirc (f \wedge m))$
}
{
\ \ \ $wcp(cecw_fc_fp_fcw_fc_fe_fcw_mc_mh_mh_mhcnncpc)^\omega$
}
{1, 3}
{201.50s}

In the sequence shown for the solution, a letter is the location of the robot, and the subscript of the letter, if any, is what the robot is carrying.  Not that this optimal solution satisfies only the first and third soft constraints.

For comparison purposes, we also executed on this example, the brute-force algorithm to compute a shortest accepting lasso that minimizes the cost function $f_{\omega}$. That algorithm failed to compute such a lasso in 15 hours.
 
\subsection{Case study: Retirement home}
In this section, we look back to the retirement home example from Section~\ref{sec:intro}.
The transition system in that example has atomic propositions for locations ---$r_1$ for common room 1, $r_2$ for common room 2, and $t$ for toy room--- and also for the robot acts ---$g$ for juggling, and $b$ for making animal balloons.
The robot is tasked to \emph{visit $r_1$, $r_2$, and $t$, each one infinitely often}.  The robot's specification also includes 6 soft constraints (some of which were mentioned in Section~\ref{sec:intro}):

\begin{enumerate}
        \item After making animal balloons in room 1, the robot should immediately do the same act in room 2.
    \item The robot should perform each of the acts in room 1 infinitely often.
    \item The robot should not perform the act of making animal balloons in room 2.
    \item If the current act in room 2 is making animal balloons, then the robot should eventually make animal balloons in that room.
    \item The robot should not stay in a common room once it performed an act.
    \item The robot should perform at least two acts in each common room once it has entered that room. Those two acts can be different or not.
\end{enumerate}

We can formalize this scenario as an instance of OLPSC:
\casestudyinternal
{$\square \lozenge r_1 \wedge \square \lozenge r_2 \wedge \square \lozenge t$}
{
\ \ \ 1) $\square ((r_1 \wedge b)\rightarrow \bigcirc ((\neg b \wedge \neg g) \:\mathcal{U} (r_2 \wedge b)))$ \\
\ \ \ 2) $\square \lozenge (r_1 \wedge g) \wedge \square \lozenge (r1 \wedge b)$ \\
\ \ \ 3) $\square (r_2 \rightarrow \neg b)$ \\
\ \ \ 4) $\square ((r_2 \wedge b) \rightarrow \lozenge (r_2 \wedge g))$ \\
\ \ \ 5) $\square (r_1 \rightarrow \bigcirc \neg r_1) \wedge  \square (r_2 \rightarrow \bigcirc \neg r_2 )$\\
\ \ \ 6) $\square ((\neg r_1 \wedge \bigcirc r_1)\rightarrow (\bigcirc \bigcirc r_1))$
}
{
\ \ \ $s_l s_{r_1,g} s_l s_{r_2,b} (s_ls_{r_2,b}s_ls_{r_2,g}s_ls_{r_2,b}s_ls_{r_1g} $\\ \ \ \ \ $s_ls_{r_1,b}s_ls_{r_2,b}s_ls_{r_1,b}s_ls_{r_2,b}s_ls_ts_ts_ls_{r_2,b}s_l)^\omega$ 
}
{1, 2, 4, 5}
{21.91s}

For comparison, the brute-force algorithm computed $r_{\mathcal{P}}^*=s_ls_{r_1,g}s_{r_1,b}s_ls_{r_2,b}s_{r_2,g}s_ls_ts_l(s_{r_1,g}s_{r_1,b}s_ls_{r_2,b}s_{r_2,g}s_ls_ts_l)^\omega$ as a shortest accepting lasso minimizing $f_\omega$ in 3,254 seconds.
The shortest accepting lasso has length 16, while the lasso generated by our algorithm has length 26. 
Now suppose there is a change in the relative ordering of the soft constraints, in which the last two are swapped.  This induces a change to the set of constraints that can be satisfied, but not to the basic structure of the product automaton.

%
%For the example above, we swapped the order of constraint (5) and (6), and then synthesized an optimal lasso on the same product automaton that was already formed. 
%
%The results are as follows:

\casestudyinternal
{$\square \lozenge r_1 \wedge \square \lozenge r_2 \wedge \square \lozenge t$}
{
\ \ \ 1) -- 4) Same as above.\\
%\ \ \ 1) $\square ((r_1 \wedge b)\rightarrow \bigcirc ((\neg b \wedge \neg g) \:\mathcal{U} (r_2 \wedge b)))$ \\
%\ \ \ 2) $\square \lozenge (r_1 \wedge g) \wedge \square \lozenge (r1 \wedge b)$ \\
%\ \ \ 3) $\square (r_2 \rightarrow \neg b)$ \\
%\ \ \ 4) $\square ((r_2 \wedge b) \rightarrow \lozenge (r_2 \wedge g))$ \\
\ \ \ 5) $\square ((\neg r_1 \wedge \bigcirc r_1)\rightarrow \bigcirc \bigcirc r_1)) $ \\
\ \ \ 6) $\square (r_1 \rightarrow \bigcirc \neg r_1 ) \wedge  \square (r_2 \rightarrow \bigcirc \neg r_2)$
}
{
\ \ \ $s_l s_{r_2,b} (s_{r_2,g} s_{r_2,b}s_{r_2,g}s_ls_{r_1,g}s_{r_1,b}s_ls_{r_2,b}s_{r_2,g}$\\ \ \ \ \ \ $s_ls_{r_1,g}s_{r_1,b}s_ls_{r_2,b}s_{r_2,g}s_ls_ts_ts_ls_{r_2,b}s_{r_2,g})^\omega$ 
}
{1, 2, 4, 5}
{0.02s (excluding product \\ automaton construction)}

In fact, if we have already computed the product automaton for the instance above, we need only to synthesize an accepting lasso, without any need to reconstruct the product automaton again.
To synthesize the new lasso, for each state
%$s$ of the SCC in which we formed the lasso,
of the product automaton, we swap the elements of vector $\mathbf{w}$ according to the new order of constraints, and then synthesize the lasso.

It took only 20 milliseconds to synthesize an optimal run on the constructed automaton, while for the first one it took 21.91 seconds, much of which was spent to form the product automaton.
We also use the brute-force algorithm to compute the shortest accepting lasso $r_{\mathcal{P}}^*$ that minimizes the $f_\omega(r_{\mathcal{P}}^*)$.
It took 110 seconds, excluding the time of product automaton construction, for the brute-force algorithm to compute the shortest accepting lasso, which was $r_{\mathcal{P}}^*=s_ls_{r_1,g}s_{r_1,b}s_ls_{r_2,b}s_{r_2,g}s_ls_ts_l(s_{r_1,g}s_{r_1,b}s_ls_{r_2,b}s_{r_2,g}s_ls_t)^\omega$, with length 16.
Observe that the length of the lasso generated by our greedy algorithm was 23 while the length of the shortest accepting lasso was 16.
The product automaton for this problem had 1440 states. 

\subsection{Experiments}
In this section, we present results of our experiment comparing our (greedy) algorithm with the brute-force algorithm.
We performed all those experiments on the same machine on which we executed our case studies.
In this experiment, we execute both the greedy algorithm and the brute-force algorithm on a large number of graphs (product automata) of different sizes that we generated randomly by the Erd\H{o}s-R{\'e}nyi model of $G(n, p)$, according to which each edge will be included in the graph with probability $p$ independent from any other edge. 
    \begin{figure*}
        \centering
        \small
        \begin{tabular}{lSSSSSSSS} 
            \toprule
            \multicolumn{1}{l}{States} &\multicolumn{2}{c}{Greedy algorithm}    &\multicolumn{3}{c}{Brute-force algorithm}    &\multicolumn{3}{c}{Approximation ratio}    \\ 
            \cmidrule(lr){2-3}\cmidrule(lr){4-6}\cmidrule(lr){7-9}
                                      & {Avg Time (Sec.)} & {Avg lasso size} & {Times of Success} & {Avg Time (Sec.)} & {Avg lasso size} & {Min} & {Max} & {Avg} \\
            \midrule
            100            & .0001 &  17.86 & 100   & 12.404   & 9.27   & 1   & 4.33 & 2.08  \\
            200            & .0002 &  21.88 & 100   & 53.91  & 10.39   & 1   & 4.33 & 2.24  \\
            300               &  0.0001  &  23.68 &  96  &  170.151 & 10.93  &  1  &  4.4 & 2.26    \\ 
            500   &   0.0001   &   28.14   &  62  & 504.799   &   11.63    &  1  &  4.14    &   2.35   \\ 
            \bottomrule
        \end{tabular}
        \caption{Results of our experiment comparing our algorithm with the brute-force algorithm. }
        \label{fig:exprTable}
            \vspace*{-20pt}  
        \end{figure*}

Figure~\ref{fig:exprTable} shows results of our experiment.
For each graph size ---100, 200, 300, and 500--- we generated 100 random graphs.
The number of edges for each graph was approximately five times the number of vertices, and for each graph, approximately 20 percent of the states were final states.
The number of soft constraints, the size of the Boolean vectors assigned to a state by $\mathbf{w}$, for each graph was 10. 
We report for each graph size, the average time to compute an accepting lasso with minimum cost by our algorithm, and also the average size of the generated lassos.
For each test, the greedy algorithm had 20 minutes to find a shortest lasso.
Figure~\ref{fig:exprTable} shows also for each graph size, the number of times the brute-force algorithm was able to compute a shortest lasso within the 20 minutes time window.
Notice that for graph size 300, the brute-force algorithm failed four times to compute a minimal lasso within that time window, and for graph size 500, it was able to compute a shortest lasso only for 62 tests.
Accordingly, we considered in computing all those averages shown for the brute-force algorithm in that figure, only those tests for which the algorithm was able to compute a solution in 20 minutes.
This means that the actual time averages for graph size 300 and 500 are higher, and probably much higher, than those shown in that figure.
The average number of constraints satisfied were ranged from 6.30 to 6.81.
f%
    \begin{figure}[t]
  \centering
  \includegraphics[width=\linewidth]{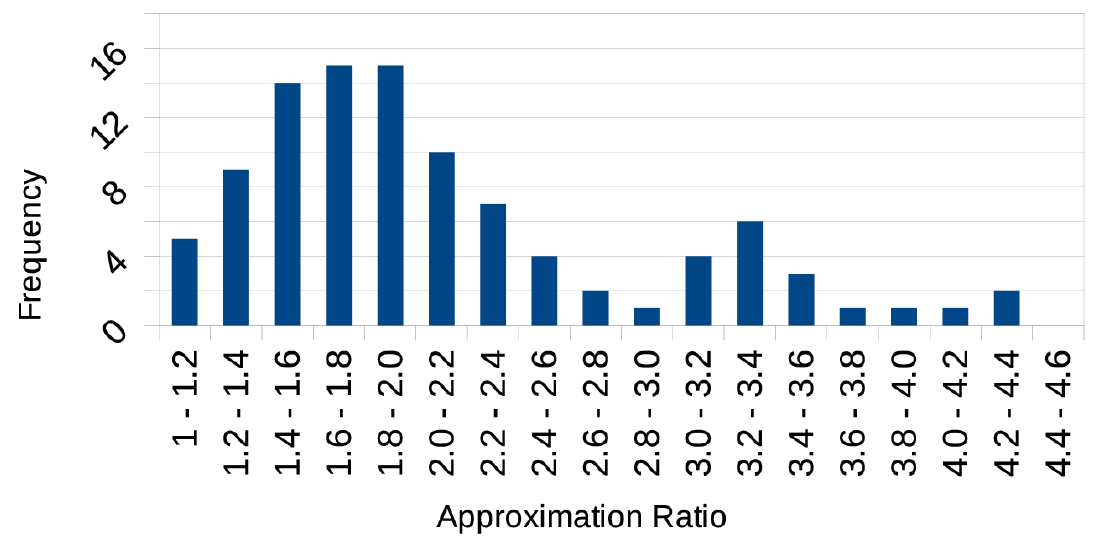}
  
  \caption{
    The distribution of approximation ratios of the lasso sizes generated by our greedy algorithm for 100 generated random graphs with 100 states.
  }
   \label{fig:dist} 
\end{figure}

    %\end{table}

For each test, we also computed approximation ratio, defined as the size of the lasso generated by the greedy algorithm over the size of the lasso generated by the brute-force algorithm.
The minimum, maximum, and the average of those ratios for each graph size is shown in Figure~\ref{fig:exprTable}.
We observe from this experiment that the greedy algorithm generates a solution significantly faster than the brute-force algorithm while the quality of solution is still acceptable. 
For graph size 100, the distribution of the approximation ratios of the 100 tests we performed is shown in Figure~\ref{fig:dist}.

We also executed a variant of our algorithm, in which we let any final state within an optimal SCC to be the midpoint and then chose a shortest one among all lassos generated for those midpoints. This algorithm increases the running time by the magnitude of the number of final states. We observed that the quality of solution is slightly improved. For the graph sizes of 100, the average of approximation ratio was 2.01 for this new variant, compared to 2.08 to original algorithm.
Because product automata are generally quite large, we may not need sacrifice computation time in favor of slightly improved quality.

%\begin{figure*}

    %\begin{table}[h!]
\vspace*{-1.5mm}
\section{Conclusions and future work}
In this paper, we considered temporal logic planning given both hard and soft specifications of the goal. 
%
%Because the size of a B\"uchi automaton is exponential to the number of atomic propositions, it is essential to model only those propositions for which we create LTL formulas.
%
%In the case where the order of constraints is subject to change, we need not to reconstruct the product automaton again; we can form the post fix part of new optimal lasso within the same SCC in which we formed the suffix part of the previous lasso.
%Inspired by~\cite{wilde2018learning}, future work can consider learning soft constraints. 
Future work can consider learning soft constraints. 
It can also consider the case where the environment is dynamic. In this case, the changes are reflected in the product automaton, for which one needs to maintain the SCCs of the automaton in a data structure that is able to quickly adapt to the changes.

\label{sec:conc}

\vspace*{-1.5mm}
\bibliography{References}
  
\end{document}